\newif\ifRAL
\RALfalse 	

\newif\ifTR
\TRfalse 	

\newif\ifPrePrint
\PrePrinttrue

\newif\ifDraft
\Draftfalse

\ifRAL
\documentclass[letterpaper, 10 pt, journal, twoside]{IEEEtran}
\else
\documentclass[letterpaper, 10 pt, conference]{ieeeconf}
\fi

\IEEEoverridecommandlockouts

\ifRAL
\else
\fi		

\makeatletter

\let\proof\@undefined
\let\endproof\@undefined
\makeatother

\usepackage{amsmath} 
\usepackage{amssymb}  
\usepackage{amsthm}
\usepackage{amsfonts}
\usepackage{mathtools}

\usepackage{bm}
\providecommand{\bm}{\pmb}

\newtheorem{thm}{Theorem}

\theoremstyle{definition}

\theoremstyle{remark}
\newtheorem*{rmk}{Remark}

\usepackage{verbatim}

\usepackage{color}
\usepackage{graphicx}

\usepackage{optidef}

\usepackage{times}

\usepackage{xspace}

\usepackage[ruled,vlined,linesnumbered]{algorithm2e}
\usepackage[noend]{algorithmic}

\graphicspath{{Images/}}

\usepackage[us]{datetime}
\usepackage{caption}
\usepackage[usenames,dvipsnames,table]{xcolor}

\usepackage{hyperref} 
\hypersetup{
    colorlinks,
    citecolor=black,
    filecolor=black,
    linkcolor=black,
    urlcolor=black,
    pdfauthor={},
    pdfsubject={},
    pdftitle={}
}

\usepackage{cleveref}

\usepackage{todonotes}

\usepackage{graphicx}

\usepackage{latexsym}
\usepackage{color}

\usepackage{cite}

\usepackage{caption}

\usepackage{tikz}
\usetikzlibrary{shapes,arrows}

\usepackage{tabularx, booktabs}
\newcolumntype{Y}{>{\centering\arraybackslash}X}
\usepackage{multirow}
\usepackage{paralist}
\usepackage{siunitx}

\newcommand{\add}[1]{#1}

\newcommand{\remove}[1]{}

\newcommand{\vect}[1]{\bm{#1}}		
\newcommand{\matr}[1]{\bm{#1}}		
\newcommand{\nR}[1]{\mathbb{R}^{#1}}		
\newcommand{\matrice}[1]{\begin{bmatrix} #1 \end{bmatrix}}	
\newcommand{\upperRomannumeral}[1]{\uppercase\expandafter{\romannumeral#1}}	

\newcommand{\vSpace}{\;\,}

\newcommand{\diag}[1]{\text{diag}\left( #1 \right)}

\DeclarePairedDelimiter{\norm}{\lVert}{\rVert} 

\newcommand{\fig}{Fig.~}	
\newcommand{\tab}{Tab.~}	
\newcommand{\sect}{Sect.~}	
\newcommand{\theo}{Theorem~}	

\newcommand{\GenFrame}{\mathcal{F}}		
\newcommand{\origin}{O}						
\newcommand{\vX}{\vect{x}}					
\newcommand{\vY}{\vect{y}}					
\newcommand{\vZ}{\vect{z}}					
\newcommand{\pos}{\vect{p}}				
\newcommand{\dpos}{\vect{v}} 
\newcommand{\ddpos}{\dot{\dpos}} 
\newcommand{\dddpos}{\ddot{\dpos}} 
\newcommand{\vZero}{\vect{0}}				
\newcommand{\eye}[1]{\matr{I}_{#1}}		

\newcommand{\frameW}{\GenFrame_W}			
\newcommand{\originW}{\origin_W}		
\newcommand{\xW}{\vX_W}				
\newcommand{\yW}{\vY_W}				
\newcommand{\zW}{\vZ_W}				

\newcommand{\frameH}{\GenFrame_{H}}			
\newcommand{\originH}{\origin_H}			
\newcommand{\xH}{\vX_H}				
\newcommand{\yH}{\vY_H}				
\newcommand{\zH}{\vZ_H}				
\newcommand{\pH}{\pos_H}			
\newcommand{\dpH}{\dpos_H}			
\newcommand{\dpHAvg}{\bar{v}_H}
\newcommand{\ddpH}{\ddpos_H}		
\newcommand{\dddpH}{\dddpos_H}		
\newcommand{\massH}{m_H}				
\newcommand{\dampingH}{\matr{B}_H}	

\newcommand{\dampingHScalar}{{b}_H}	

\newcommand{\groundForce}{ \vect{f}_g}

\newcommand{\pHRef}{{\pH^r}}

\newcommand{\humanForce}{\vect{u}_H}
\newcommand{\gravityH}{\vect{g}_H}

\newcommand{\length}{\bar{l}_{c}}					
\newcommand{\springCoeff}{k_{c}}			
\newcommand{\cableForce}{\vect{f}_{c}}			
\newcommand{\cableForcePlanar}{\vect{f}_{cxy}}			

\newcommand{\dcableForce}{\dot{\vect{f}}_{c}}			
\newcommand{\cableForceX}{{f}_{cx}}			
\newcommand{\cableForceY}{{f}_{cy}}			
\newcommand{\cableAttitude}{\vect{l}_{c}}
\newcommand{\cableAttitudeNorm}{\norm{\vect{l}_{c}}}
\newcommand{\cableAttitudeNormOf}[1]{\norm{\vect{l}_{c}(#1)}}
\newcommand{\dCableAttitude}{\dot{\vect{l}}_{c}}
\newcommand{\tension}{t_{c}(\norm{\cableAttitude})}				
\newcommand{\tensionOf}[1]{t_{c}(#1)}				
\newcommand{\cableAttitudeNormInitial}{{l}_{c0}}

\newcommand{\frameR}{\GenFrame_{R}}			
\newcommand{\originR}{O_{R}}					
\newcommand{\xR}{\vX_{R }}								
\newcommand{\yR}{\vY_{R }}								
\newcommand{\zR}{\vZ_{R }}								
\newcommand{\pR}{\pos_{R }}						
\newcommand{\dpR}{\dpos_{R }}					
\newcommand{\ddpR}{\ddpos_{R }}				
\newcommand{\dddpR}{\dddpos_{R }}				

\newcommand{\dampingA}{\matr{B}_{A}}		
\newcommand{\dampingAScalar}{{b}_{A}}		
\newcommand{\dampingAScalarPosCtrl}{{b}_{A}^{\vect{\Gamma}}}		
\newcommand{\dampingAScalarPosCtrlHVel}{{b}_{A}^{\vect{\Gamma}_H}}		
\newcommand{\inertiaA}{\matr{M}_{A}}		
\newcommand{\uA}{\vect{u}_{A}}					

\newcommand{\state}{\vect{x}}						
\newcommand{\stateSet}{\mathcal{X}}						
\newcommand{\stateEq}{\bar{\vect{x}}}						
\newcommand{\dstate}{\dot{\vect{x}}}						

\newcommand{\dynamics}[1]{\vect{f}(#1)}
\newcommand{\modelError}{\vect{\delta}}

\newcommand{\kpH}{k_{H}}	
\newcommand{\KpH}{\matr{K}_{H}}	
\newcommand{\cableForceZDes}{{f}_{z}}			
\newcommand{\cableForceRef}{{\cableForce^r}}		
\newcommand{\errorR}{\vect{e}_R}

\newcommand{\pRRef}{{\pR^r}}

\newcommand{\positionController}{\vect{\Gamma}(\state)}
\newcommand{\positionControllerSat}{\vect{\gamma}}
\newcommand{\positionControllerSatPlanar}{\vect{\gamma}_{xy}}
\newcommand{\positionControllerSatX}{{\gamma}_x}
\newcommand{\positionControllerSatY}{{\gamma}_y}
\newcommand{\positionControllerSatZ}{{\gamma}_z}

\newcommand{\positionControllerHvel}{\vect{\Gamma}_H(\state)}
\newcommand{\trajectoryPosCtrlSat}{\state^\star(t)}
\newcommand{\velPosCtrlSat}{\vect{v}^\star}

\newcommand{\lyapunovFunction}{V(\state)} 
\newcommand{\dLyapunovFunction}{\dot{V}(\state)} 

\newcommand{\dampingLyapunov}{\matr{B}}

\newcommand{\invariantSet}{\Omega_{\alpha}}		
\newcommand{\invariantSetZero}{\Omega_{0}}		
\newcommand{\maxInvariantSet}{\mathcal{M}}
\newcommand{\dVZeroSet}{\mathcal{E}}				

\newcommand{\pRerror}{{\vect{p}'_R}}
\newcommand{\dpRerror}{{{\vect{v}}'_R}}
\newcommand{\ddpRerror}{{\dot{\vect{v}}'_R}}
\newcommand{\pHerror}{{\vect{p}'_H}}
\newcommand{\dpHerror}{{{\vect{v}}'_H}}
\newcommand{\ddpHerror}{{\dot{\vect{v}}'_H}}
\newcommand{\pRPosCtrlSat}{\pR^\star}
\newcommand{\pHPosCtrlSat}{\pH^\star}

\newcommand{\stateError}{\state'}
\newcommand{\stateSetError}{\stateSet'}
\newcommand{\dstateError}{{\dstate'}}
\newcommand{\dynamicsError}[1]{\vect{f}'(#1)}
\newcommand{\lyapunovFunctionError}{V(\stateError)} 
\newcommand{\dLyapunovFunctionError}{\dot{V}(\stateError)} 

\newcommand{\expTime}{T}

\newcommand{\paperTitle}{Human-State-Aware Controller for a Tethered Aerial Robot Guiding a Human by Physical Interaction}
\ifRAL 
\author{Mike Allenspach, Yash Vyas, Matthias Rubio, Roland Siegwart, \textit{Fellow, IEEE},\\ and Marco Tognon, \textit{Member, IEEE}
	
	\thanks{Manuscript received September 09, 2021; accepted December 28, 2021. Date of publication January 18, 2022; date of current version February 2, 2022.}
	\thanks{This work was supported by ETH Career Seed Grant SEED-04 20-2 (AEROGUIDE). This paper was recommended for publication by Associate Editor W. Zhang and Editor J. Yi upon evaluation of the reviewers' comments. \textit{(Corresponding author: Mike Allenspach)}
	} 
	
    \thanks{The authors are with Autonomous Systems Lab, ETH Zurich, 8092 Switzerland (e-mail: \href{mailto:amike@ethz.ch}{amike@ethz.ch}; \href{mailto:yavyas@student.ethz.ch}{yavyas@student.ethz.ch}; \href{mailto:mrubio@student.ethz.ch}{mrubio@student.ethz.ch}; \href{mailto:rsiegwart@ethz.ch}{rsiegwart@ethz.ch}; \href{mailto:mtognon@ethz.ch}{mtognon@ethz.ch}). }
    \thanks{This letter has supplementary downloadable material available at \url{https://doi.org/10.1109/LRA.2022.3143574}, provided by the authors.}
	\thanks{Digital Object Identifier 10.1109/LRA.2022.3143574}	
    }

\else 
\author{Mike Allenspach, Yash Vyas, Matthias Rubio, Roland Siegwart and Marco Tognon
    \thanks{Mike Allenspach, Yash Vyas, Matthias Rubio, Roland Siegwart, Marco Tognon are with the Autonomous Systems Lab, ETH Zurich, 8092 Switzerland. {\tt \footnotesize \href{mailto:amike@ethz.ch}{mailto:amike@ethz.ch}, \href{mailto:mtognon@ethz.ch}{mtognon@ethz.ch}. }}
	\thanks{This work was supported by ETH Career Seed Grant SEED-04 20-2 (AEROGUIDE)}
}
\fi

\ifRAL 
\title{\paperTitle}

\else 
\title{\bf \paperTitle}
\fi

\ifPrePrint
\makeatletter
\def\ps@titlepagestyle{
	\def\@oddfoot{}\def\@evenfoot{}
	\def\@oddhead{\textcolor{red}{\sf\footnotesize Preprint version, final version at http://ieeexplore.ieee.org/ \hfill IEEE Robotics and Automation Letters 2022}}
	\def\@evenhead{\textcolor{red}{\sf\footnotesize  Preprint version, final version at http://ieeexplore.ieee.org/  \hfill IEEE Robotics and Automation Letters 2022}}%
}%
\makeatother
\makeatletter
\def\ps@headings{
	\def\@oddfoot{\textcolor{red}{\sf\footnotesize  Preprint version, final version at http://ieeexplore.ieee.org/ \hfill \thepage \;\;~\hfill~\hfill IEEE Robotics and Automation Letters 2022}}\def\@evenfoot{\hfill\thepage\hfill}
	\def\@oddhead{}\def\@evenhead{}%
}%
\makeatother
\fi	

\ifDraft
\makeatletter
\def\ps@titlepagestyle{
	\def\@oddfoot{}\def\@evenfoot{}
	\def\@oddhead{\textcolor{red}{\sf Draft version  \hfill Confidential}}
	\def\@evenhead{\textcolor{red}{\sf  Draft version  \hfill Confidential}}%
}%
\makeatother
\makeatletter
\def\ps@headings{
	\def\@oddfoot{\textcolor{red}{\sf  Draft version  \hfill Confidential}}\def\@evenfoot{\hfill\thepage\hfill}
	\def\@oddhead{}\def\@evenhead{}%
}%
\makeatother
\usepackage{draftwatermark}
\SetWatermarkText{Confidential draft}
\SetWatermarkScale{0.6}
\SetWatermarkLightness{0.9} 

\fi	

\begin{document}

\maketitle

\begin{abstract}
With the rapid development of Aerial Physical Interaction, the possibility to have aerial robots physically interacting with humans is attracting a growing interest.
\remove{Although physical Human-Robot Interaction with ground robots has been the focus of numerous studies, we cannot say the same for aerial robots.}%
In our one of our previous works~\cite{tognon2021physical}, we considered one of the first systems in which a human is physically connected to an aerial vehicle by a cable. 
There, we developed a compliant controller that allows the robot to pull the human toward a desired position  using forces only as an indirect communication-channel.
However, this controller is based on the robot-state only, which makes the system not adaptable to the human behavior, and in particular to their walking speed.
This reduces the effectiveness and comfort of the guidance when the human is still far from the desired point. 
In this \add{paper,}\remove{manuscript} we formally analyze the problem and propose a human-state-aware controller that includes a human's velocity feedback.
We theoretically prove and experimentally show that this method provides a more consistent guiding force which enhances the guiding experience.

\end{abstract}

\ifRAL 
\else 
{} 
\fi

\section{INTRODUCTION}\label{sec:intro}
\ifRAL
\IEEEPARstart{A}{erial}
\else
Aerial
\fi
Physical Interaction (APhI) is seeing a constant growth of interest. 
Motivated by scientific challenges, as well as industrial and economical potentials, the robotics community is conceiving and studying new aerial robots able to interact with and manipulate the environment~\cite{2021g-OllTogSuaLeeFra}.
\remove{Striving to perform manipulation tasks of increasing complexity, e.g., pushing, sliding~\cite{2019e-TogTelGasSabBicMalLanSanRevCorFra}, transportation and manipulation of articulated objects~\cite{2021-LeeSeoJanLeeKim}, the research community proposed various solutions to enhance physical interaction capabilities of aerial vehicles. }
Those go from new multi-rotor platforms that are fully-, omni-, and over-actuated~\cite{2021f-HamUsaSabStaTogFra}, to recent aerial manipulator~\cite{2021-BodTogSie,2020-SuaReaVegHerRodOll}. 

With the advances in manipulation capabilities of aerial robots, it is natural that the research community starts looking with interest at one of the most challenging types of physical interaction: the one with humans~\cite{2021-SelCogNikIvaSic}.
The highly complex dynamics of humans, the instability of aerial robots, and the strict requirement for safety, make \add{\textit{physical Human-Aerial Robot Interaction} (pHARI)}\remove{with aerial robots} an extremely hard to solve, and by far poorly addressed, research problem.  
Most of the works presented so far focus on contactless interaction between human and aerial robots.
These investigate \begin{inparaenum}[i)]
	\item bilateral teleoperation methods to help humans control single and multiple aerial vehicles~\cite{2014-MerStrCar,2020-LeeBalSarDeSCoeShiKimTriKon};
	\item the use of natural communication means like gestures, speech, and gaze direction to let the human interface with the robot~\cite{2017-PesHitKau,cauchard2015drone}; or
	\item the social acceptance of {aerial} robots~\cite{2017-AchBevDun}.
\end{inparaenum}
%
\begin{figure}
	\centering
	\includegraphics[width = \columnwidth]{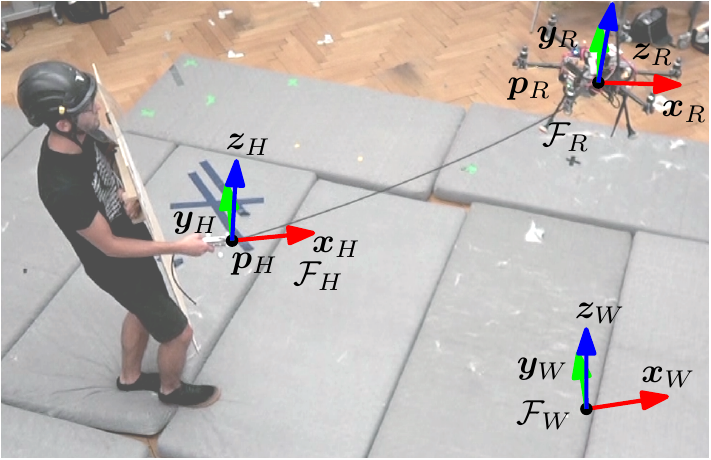}
	\caption{Representation of the aerial human-tethered guiding system.}
	\label{fig:model}
\end{figure}
\cite{2013-AugDan} presents a preliminary work addressing \add{pHARI}\remove{ for aerial vehicles}, but the interaction with the human is limited to simple contacts.
On the other hand, in~\cite{tognon2021physical} we considered one of the first systems in which a human and an aerial robot are tightly coupled by a physical connection.
Taking inspiration by the literature of tethered aerial vehicles used for physical interaction~\cite{2017a-TogFra, 2020g-TogFra}, the human holds a handle which is in turn connected to an aerial vehicle through a cable (see \fig\ref{fig:model}).
For this system, we designed an admittance-based control strategy that allows the robot to safely ``physically guide'' the human toward a desired position by exploiting the cable as a force-based communication channel.
Although the community already proposed contactless approaches for guiding humans with aerial vehicles, e.g., based on sound~\cite{2015-AviFunHen} or vision~\cite{2015-MueMui}, the use of forces as more direct feedback represents a strong novelty that \remove{particularly}suits for persons with \add{visual/hearing} limitations.

In~\cite{tognon2021physical} the robot is shown to be able to steer the human to the desired location, both in theory and real experiments. 
However, during the tests we noticed that if the human is very confident and follows the robots with a relatively high walking speed, the felt guiding force would reduce with respect to (w.r.t.) the desired one. In the extreme case, the cable can become slack, effectively making the human stop and wait for a non-zero guiding force.
This clearly represents an undesired behavior that reduces the effectiveness and comfort of the aerial guidance system.
Intuitively, this is \add{because}\remove{caused by the fact that} the controller in~\cite{tognon2021physical} is based on the robot-state only and does not consider any feedback w.r.t. the human-state.
This shows the need of more human-related feedback and adaptation laws in robot controllers for pHRI tasks.

This requirement motivated the present work in which we first formally analyze the mentioned problem, finding the mathematical reasoning behind it.
The study of the equilibria and relative stability of the system when the human is pulled by a constant force, formally confirms our first hypothesis: a reduction of the felt force by the human w.r.t. the desired one that depends on the human velocity. 
With a better definition and understanding of the problem, we propose a new human-state-aware controller aiming at a zero guiding force error, independently to human walking speed. 
In particular, we modify the controller in~\cite{tognon2021physical} to include an additional feedback w.r.t. the human velocity.
Again, by a formal study of the equilibria and stability of the system, we theoretically confirm that when the desired guiding force is constant, the force error asymptotically converges to zero.
Real experiments comparing the two strategies validate the theoretical findings.

\section{MODELING}\label{sec:modeling}
We study a system where a human and an aerial robot physically interact through a cable (see \fig\ref{fig:model}). The latter connects the vehicle to a handle held by the human.
\add{The system modeling follows} the approach in~\cite{tognon2021physical}. 
%
Considering a generic frame $\GenFrame_\star$, we define by $\origin_\star$ its origin and by $\{\vX_\star,\vY_\star,\vZ_\star\}$ its unit axes.
We then define an \textit{inertial frame}, $\frameW = \{\originW,\xW,\yW,\zW\}$ with $\zW$ opposite to gravity.

To describe the \add{human-state}\remove{ of the human}, we define a \textit{body frame},  $\frameH = \{\originH,\xH,\yH,\zH\}$, rigidly attached to the handle such that $\originH$ coincides with the point where the cable is attached. 
The position and linear velocity of the human expressed in $\frameW$ are given by the vectors $\pH \in \nR{3}$ and $\dpH  \in \nR{3}$, respectively.
Similarly, we describe the \add{aerial vehicle-state}\remove{ of the aerial vehicle} by the one of the \textit{body frame} $\frameR = \{ \originR,\xR,\yR,\zR\}$, rigidly attached to it.
Its position and linear velocity are given by the vectors $\pR\in \nR{3}$ and $\dpR \in \nR{3}$, respectively, both expressed in $\frameW$. 
We neglect the rotational part since it is not relevant for the following derivations.

As commonly done in the state of the art~\cite{1998-IkeMiz,2012-MorLawKucSezBasHir,2007-DucGos}, we approximate the human's dynamics with the one of a mass-spring-damper system:
\begin{align}
	\massH \ddpH = -\gravityH - \dampingH\dpH + \cableForce + \groundForce \add{+ \humanForce},
	\label{eqn:humanModel}
\end{align}
where $\massH \in \nR{}_{>0}$ and the symmetric positive definite matrix $\dampingH\in \nR{3 \times 3}_{>0}$ are the  mass and apparent damping, respectively; $\cableForce 
\in \nR{3}$ is the cable force applied to the human at $\originH$; $\gravityH = \massH g \zW$; $g \in \nR{}_{>0}$  is the gravitational constant; and 
$\groundForce$ is the ground reaction force. \add{Further, $\humanForce\in\nR{3}$ is the sum of all forces applied by the human that generate a translational motion.}
We assume that the forces in \eqref{eqn:humanModel} are such that $\ddpH^\top \zW = \dpH^\top \zW = 0$, i.e., the human is constrained on the ground. 

To ensure a safe physical interaction with the environment, and so with the human, a common solution consists of using an \textit{admittance control strategy} together with a position controller on the robot side~\cite{2016-HadCro}. 
This strategy allows to make the robot behave as a second order system with a specific admittance (mass, stiffness, and damping) subjected to the measured external force acting on the robot.
Under this control scheme, closed loop dynamics of the robot are:
\begin{align}\label{eqn:admittance}
	\ddpR = \inertiaA^{-1}\left( -\dampingA\dpR - \cableForce + \uA \add{+ \modelError} \right),
\end{align}
where the two symmetric positive definite matrices $\inertiaA,  \dampingA \in \nR{3 \times 3}_{>0}$ are the virtual inertia and damping of the robot.
$\uA \in \nR{3}$ is an additional input that will be defined in the following to achieve the desired control goal.
\add{$\modelError \in \nR{3}$ is a bounded variable that takes into account all the tracking errors due to model mismatches and uncertainties, input and state limits, external disturbances, and estimation errors.}

Finally, we model the cable as a unilateral spring along its principal direction. Following the Hooke's law, we write the cable force applied to $\originH$ as
\begin{align}\label{eqn:cableForce}
	\cableForce &= \tension {\cableAttitude}/{\cableAttitudeNorm}\\
	\tension &= \begin{cases}
					\springCoeff(\cableAttitudeNorm - \length)& \text{ if } \cableAttitudeNorm - \length > 0 \\
					0	&	\text{ otherwise }
				\end{cases},	
\end{align}
where $\tension$ represents the cable internal force intensity, $\cableAttitude = \pR - \pH$, $\length \in \nR{}_{> 0}$ is the constant nominal length of the cable, and $\springCoeff \in \nR{}_{> 0}$ is the constant elastic coefficient. 
The force produced at the other end of the cable, namely to the robot at $\originR$, is equal to $-\cableForce$ as shown in \eqref{eqn:admittance}.

To write the full system dynamics, i.e., human and controlled robot connected by the cable, we define the system state vector as $\state = [\pH^\top \vSpace \dpH^\top  \vSpace \pR^\top  \vSpace \dpR^\top]^\top \in \stateSet \subset \nR{12}$. Considering~\eqref{eqn:humanModel} and~\eqref{eqn:admittance}, we can write the state dynamics as $\dstate = \dynamics{\state,\uA}$, where:
%
\begin{align}
	\dynamics{\state,\uA} &= \matrice{\dpH \\  \frac{1}{\massH}\left(-\gravityH - \dampingH\dpH + \cableForce  +  \groundForce  \right) \\ \dpR \\ \inertiaA^{-1}\left( -\dampingA\dpR - \cableForce + \uA \right) },
	\label{eqn:closedLoopSystemDyn}
\end{align}
and $\cableForce$ is computed as in \eqref{eqn:cableForce}. 
\add{For the sake of deriving the control law and proving the asymptotic stability in nominal conditions, we consider zero human forces $\humanForce=\vZero$ and modeling errors $\modelError=\vZero$. In \sect\ref{ssec:robustness} we shall discuss the robustness of the method when these conditions are not met.} 


\section{PROBLEM FORMULATION AND ANALYSIS}\label{sec:control}
In~\cite{tognon2021physical} we already designed a robot's position-based control law $\positionController$ that makes the robot able to steer the human toward a desired position $\pHRef \in \nR{3}$ (such that $\zW^\top\pHRef = 0$).
In the next subsection we recall the main results and highlight encountered problems that are then formally analyzed and addressed by a new controller.

\subsection{Equilibria and Stability Analysis under \texorpdfstring{$\positionController$}{s}}

\subsubsection{Steady state solution}
In regular conditions, the following theorem holds:
\begin{thm}\label{thm:stability}
	Given the desired human position $\pHRef \in \nR{3}$, such that $\zW^\top\pHRef = 0$, we consider the control law $\positionController$\remove{ such that}
	\begin{align}
		\uA = \positionController := \KpH \errorR + \cableForceRef,
		\label{eqn:controller}
	\end{align}
	where $\KpH = \diag{\kpH,\kpH,0}$, $\kpH \in \nR{}_{> 0}$, $\errorR = \pRRef - \pR$ is the robot position error, $\pRRef$ is the robot \add{reference} position, and $\cableForceRef \in \nR{3}$ is a constant forcing input, defined as follow 
	\begin{align}
		\cableForceRef &= \cableForceZDes\zW \label{eqn:regulation:cableForceRef} \\
		\pRRef &= \pHRef + \cableForceRef \left( \frac{1}{\springCoeff} + \frac{\length}{\norm{\cableForceRef}} \right),\label{eqn:regulation:pRRef}
	\end{align}
	with $\cableForceZDes \in \nR{}$.
	Then, \remove{considering the system \eqref{eqn:closedLoopSystemDyn} under the control law  \eqref{eqn:controller},}
	the zero velocity equilibrium $\stateEq = [\pHRef^\top \vSpace \vZero^\top \vSpace \pRRef^\top \vSpace \vZero^\top]^\top \in \stateSet$ \add{of the system \eqref{eqn:closedLoopSystemDyn} under the control law \eqref{eqn:controller}} is asymptotically stable if $0 < \cableForceZDes < \massH g$.
\end{thm}

Thanks to the proof available in~\cite{tognon2021physical}, we know that at steady state the system will converge to $\stateEq$, i.e., the human will be in the desired position. 
Nevertheless, during experiments we noticed an undesired behavior during transient, in particular when the human is far from the desired position. 
While the human walks following the cable force provided by the robot, the latter is less than the desired one. 
Moreover, the force error increases with the human speed increasing.
This makes the guidance less effective. 
In extreme cases, the force might be very close to zero or even be zero, forcing the human to stop and wait for the tension on the cable to be re-established.

\subsubsection{Steady state solution under constant control input}
To better characterize this phenomena, we consider the case in which the human is far from the desired position. In this condition, we saturate $\errorR$ in \eqref{eqn:controller} to avoid extreme forces applied to the human. 
It results that $\positionController = \positionControllerSat$ where $\positionControllerSat\add{=[{ \positionControllerSatX \vSpace \positionControllerSatY \vSpace \positionControllerSatZ}]^\top}\in\nR{3}$ is a constant value that depends on the imposed error saturation.  
In the next theorem, we characterize the equilibria and its stability for the system~\eqref{eqn:closedLoopSystemDyn} subjected to the constant control input $\uA = \positionControllerSat$.   

\begin{thm}\label{thm:stability:posCtrlSat}
	Under the condition $\uA = \positionController = \positionControllerSat$ constant, the constant velocity trajectory $\trajectoryPosCtrlSat$ where 
	\begin{align}
		\dpH(t) = \dpR(t) = \velPosCtrlSat := \left(\dampingA+\dampingH\right)^{-1}\add{\matrice{\positionControllerSatX \vSpace \positionControllerSatY\vSpace 0}^\top}, \label{eqn:equilibria:positionControllerSat:velocity}
	\end{align}
	\remove{for every }$\forall t \in \nR{}_{\geq 0}$, is asymptotically stable for the system \eqref{eqn:closedLoopSystemDyn}.
\end{thm}
\begin{proof}
	We firstly have to show that the trajectory $\trajectoryPosCtrlSat$ 
	is of equilibrium for the system~\eqref{eqn:closedLoopSystemDyn} under the condition $\uA(t) = \positionControllerSat$.
	At equilibrium, accelerations must be indefinitely zero, i.e., $\ddpR = \ddpH = \dddpR = \dddpH = \vZero$.  
	Applying this condition to~\eqref{eqn:closedLoopSystemDyn} and its derivative, the following must be true:
	\begin{align}
		\vZero &= \add{-\gravityH} - \dampingH\dpH + \cableForce \add{+ \groundForce} \label{eqn:equilibriaCond:ddpH}\\
		\vZero &= - \dampingA\dpR - \cableForce + \positionControllerSat \label{eqn:equilibriaCond:ddpR}\\
		\vZero &= \dcableForce. \label{eqn:equilibriumConditions:dcableForce}
	\end{align}

	Firstly notice that $\cableForce$ cannot be zero. In fact, if $\cableForce = \vZero$ then, from \eqref{eqn:equilibriaCond:ddpH} and \eqref{eqn:equilibriaCond:ddpR}, $\dpH = \vZero$ and $\dpR = \dampingA^{-1}\positionControllerSat$. 
	Since $\dpH = \vZero$ and $\dpR \neq 0$, $\dCableAttitude = \dpR$ and $\norm{\dCableAttitude} = \norm{\dpR} \neq 0$. 
	This means that the cable will stretch and eventually will become taut, i.e., $\norm{\cableAttitude} > \length$ thus $\cableForce \neq \vZero$.
	 
	Let us now consider~\eqref{eqn:equilibriumConditions:dcableForce} and calculate the time derivative of~\eqref{eqn:cableForce}. We obtain:
	\begin{align}
		\vZero &= \springCoeff\frac{\cableAttitude}{\norm{\cableAttitude}}\left(\frac{d}{dt}\norm{\cableAttitude}\right) +  \springCoeff(\cableAttitudeNorm - \length)\left(\frac{d}{dt} \frac{\cableAttitude}{\norm{\cableAttitude}} \right)  \\
			&= \frac{\cableAttitude}{\norm{\cableAttitude}}\frac{\cableAttitude^\top}{\norm{\cableAttitude}}\dCableAttitude + (\cableAttitudeNorm - \length) \frac{1}{\norm{\cableAttitude}}\left( \eye{} - \frac{\cableAttitude\cableAttitude^\top}{\norm{\cableAttitude}^2} \right) \dCableAttitude. \label{eqn:equilibriaCond:dcableForce:partial:1}
	\end{align}
	To simplify the analysis of~\eqref{eqn:equilibriaCond:dcableForce:partial:1}, we define the variables $c = \length/\norm{\cableAttitude}$ and $\matr{A} = {\cableAttitude\cableAttitude^\top}/{\norm{\cableAttitude}^2}$. We can then rewrite~\eqref{eqn:equilibriaCond:dcableForce:partial:1} as
	\begin{align}
		\vZero &= \left[c\matr{A} + (1 -c)\eye{} \right]\dCableAttitude. \label{eqn:equilibriaCond:dcableForce:partial:2}
	\end{align}
	We can show that~\eqref{eqn:equilibriaCond:dcableForce:partial:2} holds only if $\dCableAttitude = \vZero$, i.e., if $\dpH = \dpR$. In fact, it cannot be that $c\matr{A} + (1 -c)\eye{} = \vZero$ because $\matr{A}$ cannot be diagonal.
	
	Finally, replacing the condition $\dpH = \dpR$ into~\eqref{eqn:equilibriaCond:ddpH}, and summing both sides of~\eqref{eqn:equilibriaCond:ddpH} and~\eqref{eqn:equilibriaCond:ddpR} \add{under the restriction that the human is constrained to the ground,} we obtain~\eqref{eqn:equilibria:positionControllerSat:velocity}.
	
	We are now ready to prove the asymptotic stability of $\trajectoryPosCtrlSat$.
	To do so we apply the following change of coordinates
	\begin{align}
		 \begin{array}{ll}
		 	\pRerror = \pR - \pRPosCtrlSat & \dpRerror = \dpR - \velPosCtrlSat \\
		 	\pHerror = \pH - \pHPosCtrlSat & \dpHerror = \dpH - \velPosCtrlSat,
		 \end{array}
	\end{align}
	where $\pRPosCtrlSat(t) = \pR(0) + \velPosCtrlSat \cdot t$ and $\pHPosCtrlSat(t) = \pH(0) + \velPosCtrlSat \cdot t$. 
	Considering the new state vector $\stateError = [\pHerror^\top \vSpace \dpHerror^\top  \vSpace \pRerror^\top  \vSpace \dpRerror^\top]^\top \in \stateSetError \subset \nR{12}$, its dynamics is $\dstateError = \dynamicsError{\stateError}$ where
	\begin{align}
		\dynamicsError{\stateError} = \matrice{\dpHerror \\  \frac{1}{\massH}\left(-\gravityH - \dampingH\dpHerror + \cableForce  +  \groundForce - \dampingH\velPosCtrlSat  \right) \\ \dpRerror \\ \inertiaA^{-1}\left( -\dampingA\dpRerror - \cableForce + \positionControllerSat - \dampingA\velPosCtrlSat \right) }. \label{eqn:closedLoopSystemDynError}
	\end{align}
	It is easy to verify that $\stateError = \vZero$ is of equilibrium, i.e.,  $\dynamicsError{\vZero} = \vZero $. Proving the asymptotic stability of $\vZero$ for \eqref{eqn:closedLoopSystemDynError} is equivalent to prove the asymptotic stability of $\trajectoryPosCtrlSat$ for \eqref{eqn:closedLoopSystemDyn} with $\uA = \positionControllerSat$.
	Therefore let us consider the following Lyapunov function:
	\begin{align}
		\lyapunovFunctionError = V_1(\stateError) + V_2(\stateError) + V_0, 
		\label{eqn:lyapunovFunction:posCtrlSat:error}
	\end{align}
	where
	\begin{align}
		V_1(\stateError) &=  \frac{1}{2}\left( \massH\dpHerror^\top\dpHerror +  \dpRerror^\top \inertiaA \dpRerror \right) \\
		V_2(\stateError) &= \int_{\cableAttitudeNormInitial}^{\cableAttitudeNormOf{t}} \tensionOf{\tau} d\tau   -\cableAttitude^\top \dampingH\velPosCtrlSat,
	\end{align}
	and $V_0$ is a constant such that $V(\vZero) = 0$. We defined $\cableAttitudeNormInitial = \cableAttitudeNormOf{0}$.
	Similarly as done in~\cite{tognon2021physical}, we can show that $\lyapunovFunctionError$ is continuously differentiable and radially unbounded. Furthermore, it is positive definite with unique global minimum in $\vZero$, such that $V(\vZero) = 0$.
	Here we skip the derivations for brevity.
	Again following similar steps, we can compute the time derivative of~\eqref{eqn:lyapunovFunction:posCtrlSat:error} that results to be:
	\begin{align}
	\begin{split}
		\dLyapunovFunctionError &= \dpHerror^\top \left( -\gravityH - \dampingH\dpHerror + \cableForce + \groundForce - \dampingH\velPosCtrlSat \right) \\
		& + \dpR^\top \left(  -\dampingA\dpRerror - \cableForce + \positionControllerSat - \dampingA\velPosCtrlSat \right) \\
		& + \dCableAttitude^\top\cableForce - \dCableAttitude^\top\dampingH\velPosCtrlSat .
	\end{split}  
	\end{align}
	Noticing that $\dCableAttitude = \dpRerror - \dpHerror$, after few simple algebraic steps we get:
	\begin{align}
		\dLyapunovFunctionError = - \dpHerror^\top\dampingH\dpHerror - \dpRerror^\top\dampingA\dpRerror.
	\end{align}
	$\dLyapunovFunctionError$ is clearly negative semidefinite in $\stateSetError$.
	
	Since $\dLyapunovFunctionError$ is only negative semidefinite, to prove the asymptotic stability we rely on \textit{LaSalle's invariance principle}~\cite{2002-Kha}. 
	The Lyapunov candidate $\lyapunovFunctionError$ is a continuously differentiable function.
	Let us define a positively invariant set $\invariantSet = \{\stateError \in \stateSetError \; | \; \lyapunovFunctionError \leq \alpha \text{ with } \alpha\in\nR{}_{>0}\}$.
	By construction, $\invariantSet$ is compact since~\eqref{eqn:lyapunovFunction:posCtrlSat:error} is radially unbounded. Notice that $\invariantSetZero$ contains only $\vZero$.
	Then we need to find the largest invariant set $\maxInvariantSet$ in $\dVZeroSet = \{ \stateError \in \invariantSet \; | \; \dLyapunovFunctionError = 0\}$. 
	A trajectory $\stateError(t)$ belongs identically to $\dVZeroSet$ if $\dot{V}(\stateError(t)) \equiv 0 \Leftrightarrow 
	\dpHerror(t) = \dpRerror(t) \equiv \vZero \Leftrightarrow
	\ddpHerror = \ddpRerror = \vZero$. 
	As we saw previously in the calculation of the equilibria, this is verified only if $\stateError = \vZero$.
	Therefore $\maxInvariantSet$ contains only $\vZero$.
	All conditions of LaSalle's principle are satisfied and $\vZero$ is globally asymptotically stable. 
	Looking back at the original system~\eqref{eqn:closedLoopSystemDyn} with $\uA = \positionController = \positionControllerSat$, we can conclude that the constant velocity trajectory $\trajectoryPosCtrlSat$ where $\dpH(t) = \dpR(t) = \velPosCtrlSat$ for every $t$, is asymptotically stable.
\end{proof}

\begin{rmk}
	When $\uA = \positionController = \positionControllerSat$, from \theo\ref{thm:stability:posCtrlSat}, we can compute the force that eventually is applied to the human. 
	Replacing the condition~\eqref{eqn:equilibria:positionControllerSat:velocity} into \eqref{eqn:closedLoopSystemDyn} we obtain:
	%
	\begin{align}
		\cableForce = \positionControllerSat - \dampingA\velPosCtrlSat,
		\label{eqn:cableForce_posCtrlSat}
	\end{align}
	which implies $\norm{\cableForce} < \norm{\positionControllerSat}$\add{, i.e.,} the human feels a lower force w.r.t. the desired one which reduces the effectiveness of the guiding strategy. 
	Moreover, the more the human walks fast (modeled as a low damping $\dampingH$), the less the force they feel on the cable.
Intuitively, looking at \eqref{eqn:admittance}, the problem comes from the fact that the robot wants to apply the desired force $\uA$ but trying to stay still (due to damping term). 
While the human walks, the two actions are opposite and the system converges to a state where the two are in equilibrium, but none of the two objectives are fully satisfied.
\end{rmk}

To avoid this problem, the robot must consider the human's speed while performing the guiding task. 
Along this line, in the next subsection we propose a modification of~\eqref{eqn:controller}, $\positionControllerHvel$, and we analyze the respective stability properties.

\section{CONTROL}

We propose to modify the control action \eqref{eqn:controller} as follow:
\begin{align}
	\uA = \positionControllerHvel =  \positionController + \dampingA\dpH,
	\label{eqn:controllerVelFeedback}
\end{align}
such that, replacing~\eqref{eqn:controllerVelFeedback} into \eqref{eqn:admittance} yields new robot dynamics:
\begin{align}\label{eqn:admittanceVelFeedback}
	\ddpR = \inertiaA^{-1}\left( \dampingA(\dpH - \dpR) - \cableForce + \positionController \right).
\end{align}
Notice that the latter includes a control action proportional to the velocity difference between human and robot.

\subsection{Equilibria and Stability Analysis under \texorpdfstring{$\positionControllerHvel$}{s}}
To keep a symmetry with the previous section, we first shall study the steady state solution of the system under $\positionControllerHvel$ for regular conditions and verify that $\stateEq$ is still asymptotically stable.
We then consider the case when $\positionController = \positionControllerSat$ to show that $\positionControllerHvel$ solves the problem of reduced guiding force, which in turn results in a more continuous guidance.
%

\subsubsection{Steady state solution}

Given the new controller~\eqref{eqn:controllerVelFeedback}, in the following theorem we verify that the desired position for the human is still asymptotically stable.

\begin{thm}\label{thm:stability:velFeedback}
	Given the human position reference $\pHRef \in \nR{3}$ such that $\zW^\top\pHRef = 0$ \add{and $\pRRef$ as in \eqref{eqn:regulation:pRRef}}, the zero velocity equilibrium $\stateEq = [\pHRef^\top \vSpace \vZero^\top \vSpace \pRRef^\top \vSpace \vZero^\top]^\top \in \stateSet$ is asymptotically stable for the system \eqref{eqn:closedLoopSystemDyn} under the control law  \eqref{eqn:controllerVelFeedback}, if\footnote{Given a generic matrix $\matr{M}$, we use the standard notation $\matr{M} \succ 0$ if $\matr{M}$ is positive definite.} $\dampingH - \frac{1}{4}\dampingA \succ 0$.
\end{thm}
\begin{proof}
	Let us first compute the equilibrium state.
	Plugging \eqref{eqn:controllerVelFeedback} into \eqref{eqn:closedLoopSystemDyn} yields the final closed-loop dynamics:
	\begin{align}
		\dynamics{\state} &= \matrice{\dpH \\  \frac{1}{\massH}\left(-\gravityH - \dampingH\dpH + \cableForce +  \groundForce\right) \\ \dpR \\ \inertiaA^{-1}\left( \dampingA(\dpH - \dpR) - \cableForce + \KpH \errorR + \cableForceRef \right) }.
		\label{eqn:closedLoopSystemDyn:regulation}
	\end{align}
	Imposing the stability condition into \eqref{eqn:closedLoopSystemDyn:regulation}, i.e., $\dstate = \dynamics{\state} =\vZero$, it is easy to verify that $\stateEq$ is the only solution. 
	In fact, noticing that at the equilibrium $\dpH = \vZero$, we can follow the same derivations as in~\cite{tognon2021physical}, not reported here for brevity.  
	
	To study the stability of $\stateEq$ we employ Lyapunov theory considering the following Lyapunov function:
	\begin{align}
		\lyapunovFunction = V_3(\state) + V_4(\state) + V_0, 
		\label{eqn:lyapunovFunction:regulation}
	\end{align}
	where,
	\begin{align}
		V_3(\state) &=  \frac{1}{2}\left( \massH\dpH^\top\dpH +  \dpR^\top \inertiaA \dpR  + \errorR^\top \KpH \errorR \right) \\
		V_4(\state) &= \int_{\cableAttitudeNormInitial}^{\cableAttitudeNormOf{t}} \tensionOf{\tau} d\tau   -\cableAttitude^\top\add{\cableForceRef} \\
		V_0 &= \cableForceZDes^2/\springCoeff + 2 \length \cableForceZDes.
	\end{align}
	In~\cite{tognon2021physical} we showed that $\lyapunovFunction$ is continuously differentiable and radially unbounded. Furthermore, it is positive definite with unique global minimum in $\stateEq$, such that $V(\stateEq) = 0$.
	
	Following similar steps as in~\cite{tognon2021physical}, we can compute the time derivative of~\eqref{eqn:lyapunovFunction:regulation}, that results to be:
	\begin{align}
		\begin{split}
		\dLyapunovFunction =& \dpH^\top \left( -\gravityH - \dampingH\dpH + \cableForce + \groundForce \right) \\
			& + \dpR^\top \left(  \dampingA(\dpH - \dpR) - \cableForce + \KpH \errorR + \cableForceRef \right) \\
			& - \dpR^\top \KpH \errorR + \dCableAttitude^\top\cableForce - \dCableAttitude^\top\cableForceRef.
		\end{split}  
	\end{align}
	After some simplifications and rearranging some terms we get the following quadratic form:
	\begin{align}
		\dLyapunovFunction = -\matrice{\dpH^\top & \dpR^\top} \underbrace{\matrice{\dampingH & -\frac{1}{2}\dampingA \\ -\frac{1}{2}\dampingA & \dampingA }}_{\textstyle\dampingLyapunov} \matrice{\dpH \\ \dpR}.
	\end{align}
	For $\dLyapunovFunction$ to be negative semidefinite in $\stateSet$, $\dampingLyapunov \in \nR{6 \times 6}$ must be positive definite.
	Using the Schur complement method we can rewrite $\dampingLyapunov$ as
	\begin{align}
		\dampingLyapunov = 	\matrice{\eye{} & -\frac{1}{2}\eye{} \\ \vZero & \eye{}}
							\matrice{\dampingH - \frac{1}{4}\dampingA & \vZero \\ \vZero & \dampingA}\matrice{\eye{} & -\frac{1}{2}\eye{} \\ \vZero & \eye{}}^\top.
		\label{eqn:dampingLyapunov:schurComplement}
	\end{align}
	Recalling that $\dampingH \succ 0$ and $\dampingA \succ 0$, it is clear that $\dampingLyapunov\succ 0$ ($\dLyapunovFunction \leq 0$ in $\stateSet$) if $\dampingH - \frac{1}{4}\dampingA \succ 0$.
	Since $\dLyapunovFunction$ is only negative semidefinite, to conclude the prove we can rely on the {LaSalle's invariance principle} as done in~\cite{tognon2021physical}. 
\end{proof}

	In the common case in which $\dampingH$ and $\dampingA$ are diagonal matrices, we can define a more specific \add{stability condition}:
	\begin{thm}\label{thm:stability:velFeedback:diagonalDamping}
		Let us assume that $\dampingH = \dampingHScalar\eye{}$ and $\dampingA = \dampingAScalar\eye{}$ where $\dampingHScalar \in \nR{}_{>0}$ and $\dampingAScalar \in \nR{}_{>0}$.
		Given the human position reference $\pHRef \in \nR{3}$ such that $\zW^\top\pHRef = 0$ \add{and $\pRRef$ as in \eqref{eqn:regulation:pRRef}}, the zero velocity equilibrium $\stateEq = [\pHRef^\top \vSpace \vZero^\top \vSpace \pRRef^\top \vSpace \vZero^\top]^\top \in \stateSet$ is asymptotically stable for the system \eqref{eqn:closedLoopSystemDyn} under the control law  \eqref{eqn:controllerVelFeedback}, if $0 < \dampingAScalar < {4}\dampingHScalar$.
	\end{thm}
	\begin{proof}
	Under the assumption $\dampingH = \dampingHScalar\eye{}$ and $\dampingA = \dampingAScalar\eye{}$, then $\dampingH - \frac{1}{4}\dampingA = (\dampingHScalar - \frac{1}{4}\dampingAScalar)\eye{}$. 
	Therefore, $\dampingH - \frac{1}{4}\dampingA \succ 0$ if $\dampingAScalar < {4}\dampingHScalar$.
	In other words, if $\dampingAScalar < {4}\dampingHScalar$, the conditions of  \theo\ref{thm:stability:velFeedback} are satisfied so we can conclude that $\stateEq$ is asymptotically stable.
	\end{proof}
	
	Being $\dampingA$ in \eqref{eqn:admittanceVelFeedback} a sort of derivative gain, Theorems~\ref{thm:stability:velFeedback} and \ref{thm:stability:velFeedback:diagonalDamping} give some upper limits for its tuning.
	In particular, \theo\ref{thm:stability:velFeedback:diagonalDamping} says that the robot must react sufficiently slower than the human. 
	\add{In practice, we notice that $\dampingAScalar$ cannot be too low neither, as a minimum dissipation is required.}
	\remove{In practice, we notice that $\dampingAScalar$ cannot be too low neither. 
	In fact, a minimum dissipation is required when $\dpH$ is close to zero.}

\subsubsection{Steady state solution under constant control input}

We consider the case in which the human is sufficiently far from the desired point such that the position error $\errorR$ is saturated, i.e., $\positionController = \positionControllerSat$.
We shall show that the new controller $\positionControllerHvel$ solves the problem of reduced guiding force.

\begin{thm}\label{thm:stability:posCtrlSat:Hvel}
	Under the condition $\positionController = \positionControllerSat$ and $\uA = \positionControllerHvel=\positionControllerSat +\dampingA\dpH$, the constant velocity trajectory $\trajectoryPosCtrlSat$\remove{ where} 
	\begin{align}
		\dpH(t) = \dpR(t) = \velPosCtrlSat := \dampingH^{-1}\add{\matrice{\positionControllerSatX \vSpace \positionControllerSatY\vSpace 0}^\top}, \label{eqn:equilibria:positionControllerHvelSat:velocity}
	\end{align}
	\remove{for every }$\forall t\in\nR{}_{\geq 0}$, is asymptotically stable if $\dampingH - \frac{1}{4}\dampingA \succ 0$.
\end{thm}
\begin{proof}
	Similarly to the proof of \theo\ref{thm:stability:posCtrlSat}, we first apply the equilibria condition
	\begin{align}
		\vZero &= \add{-\gravityH} - \dampingH\dpH + \cableForce \add{+ \groundForce} \label{eqn:Hvel:equilibriaCond:ddpH}\\
		\vZero &=  \dampingA(\dpH - \dpR) - \cableForce + \positionControllerSat \label{eqn:Hvel:equilibriaCond:ddpR}\\
		\vZero &= \dcableForce. \label{eqn:Hvel:equilibriumConditions:dcableForce}
	\end{align}
	As before, \eqref{eqn:Hvel:equilibriumConditions:dcableForce} implies $\dpR = \dpH$. 	
	Then, replacing the condition $\dpH = \dpR$ into~\eqref{eqn:Hvel:equilibriaCond:ddpH}, and summing both sides of~\eqref{eqn:Hvel:equilibriaCond:ddpH} and~\eqref{eqn:Hvel:equilibriaCond:ddpR} \add{given that the human is constrained to the ground,} we obtain~\eqref{eqn:equilibria:positionControllerHvelSat:velocity}.
	
	To proof the asymptotic stability of $\trajectoryPosCtrlSat$ we follow the same procedure as in the proof of \theo\ref{thm:stability:posCtrlSat}. 
	We apply the same change of coordinates, $\stateError$, and we consider the Lyapunov function $\lyapunovFunctionError$ in~\eqref{eqn:lyapunovFunction:posCtrlSat:error}, which is continuously differentiable, positive definite and with unique global minimum in $\vZero$, such that $V(\vZero) = 0$.
	Following the same derivations, the time derivative results
	\begin{align}
		\dLyapunovFunctionError = -\matrice{\dpHerror^\top & \dpRerror^\top} \dampingLyapunov \matrice{\dpHerror \\ \dpRerror},
	\end{align}
	where $\dampingLyapunov$ as in \eqref{eqn:dampingLyapunov:schurComplement}. 
	Therefore, $\dLyapunovFunctionError$ is negative semidefinite if $\dampingH - \frac{1}{4}\dampingA \succ 0$. 
	The proof concludes applying the {LaSalle's invariance principle} as in  previous theorems.
\end{proof}
\begin{rmk}
    As before, we can compute the steady-state cable force associated with the constant velocity trajectory. Under the conditions of \theo\ref{thm:stability:posCtrlSat:Hvel}, \eqref{eqn:Hvel:equilibriaCond:ddpR} results:
    \begin{align}
        \cableForce = \positionControllerSat.
        \label{eqn:cableForce_posCtrlHVelSat}
    \end{align}
    This confirms that the new controller $\positionControllerHvel$ is capable of rendering the desired guidance force, independent of the human damping $\dampingH$, i.e., the human walking speed.
\end{rmk}

\add{\subsection{Robustness Analysis under \texorpdfstring{$\positionControllerHvel$}{s}}
\label{ssec:robustness}
In the presence of non-zero human forces $\humanForce\neq\vZero$ and tracking errors $\modelError\neq\vZero$, the closed-loop system dynamics in \eqref{eqn:closedLoopSystemDyn} become:
\begin{align}
	\dynamics{\state,\uA} &= \matrice{\dpH \\  \frac{1}{\massH}\left(-\gravityH - \dampingH\dpH + \cableForce  +  \groundForce + \humanForce \right) \\ \dpR \\ \inertiaA^{-1}\left( -\dampingA\dpR - \cableForce + \uA + \modelError \right) }.
	\label{eqn:closedLoopSystemDynPassivity}
\end{align}
Similarly to \cite{tognon2021physical}, the presented stability proofs can be extended to show the system's passivity w.r.t. these effects. 
\begin{thm}\label{thm:passivity:regulation}
System \eqref{eqn:closedLoopSystemDynPassivity} under control law $\positionControllerHvel$ in \eqref{eqn:controllerVelFeedback} is output-strictly passive w.r.t. the storage function \eqref{eqn:lyapunovFunction:regulation} and the input-output pair $(\vect{u},\vect{y})=([{\dpH^\top \vSpace \dpR^\top}]^\top,[{\humanForce^\top \vSpace \modelError^\top}]^\top)$, if $\dampingH-\frac{1}{4}\dampingA\succ 0$.
\end{thm}
\begin{proof}
    The proof follows the same steps as shown in \cite{tognon2021physical} and the formal derivation is neglected here for brevity.
\end{proof}

}

\section{EXPERIMENTAL VALIDATION}\label{sec:experimental_validation}
\add{Indoor flight experiments are performed to confirm the theoretical conclusions about \add{the steady state performance of} the presented control law $\positionControllerHvel$ \add{under constant control input} and compare it against $\positionController$.
Due to lack of space, the validation of the asymptotic stability and robustness properties are shown only in the attached video.}

\subsection{Experimental Setup}
The employed system is shown in \fig\ref{fig:model}, consisting of a hexacopter aerial vehicle that is connected by a cable to a handle held by the human.
The handle is equipped with reflective markers, allowing to obtain a measure of the human position $\pH$ using a motion tracking system. 
Based on this, the human velocity $\dpH$ required by the controller $\positionControllerHvel$ in \eqref{eqn:controllerVelFeedback} is numerically computed using a Savitzky-Golay filter.
\remove{Notice that the velocity estimation method can easily be adapted based on the sensor suit available (e.g., using cameras onboard the robot) without having to change the controller itself.}

The cable connecting the human-handle with the aerial vehicle has a rest-length $\length=\SI{1.5}{\metre}$ and a negligible mass (less than \SI{10}{\gram}).
A measure of the cable force $\cableForce$ is obtained at $\SI{800}{\hertz}$ with a 6-axis Rokubi force-torque-sensor mounted between the cable and the drone.
The robot itself is equipped with an onboard PC, running the admittance filter together with a MPC position controller \cite{2012-KamStaAleSie}, effectively rendering the closed loop dynamics \eqref{eqn:admittance}\add{, while respecting the UAV's actuation limits}.
Position and velocity estimates $\pR,\dpR$ required by the control framework 
are obtained fusing pose measurements from the motion capture system with onboard IMU data (accelerometer and gyroscopes) at $\SI{100}{\hertz}$ in an EKF-based state estimator.

The described setup was used to collect validation datasets for three different human subjects, each being guided to a desired location \add{and instructed to follow the guiding force as much as possible, i.e., $\humanForce\approx\vZero$.}
Since the control performance depends on the human velocity $\dpH$, data is recorded at different walking speeds.
To ensure a consistent average walking speed $\dpHAvg=\frac{1}{T}\int_0^\expTime\norm{\dpH(t)}_2dt$ across all subjects during the duration $\expTime$ of the experiment, oral feedback is continuously provided to increase, decrease or maintain their current speed.
Specifically, experiments with $\dpHAvg=\SI{0.075}{\metre\per\second}$, $\dpHAvg=\SI{0.1}{\metre\per\second}$ and $\dpHAvg=\SI{0.2}{\metre\per\second}$ were conducted.
Although these are rather slow speeds, this allows to reach and maintain steady-state for a representative amount of time, despite the limited room size.

For each experiment, the $\positionController$ and $\positionControllerHvel$ control law was subsequently implemented and tested. 
Since we assume $\dpH^\top\zW=\vZero$ however, the two controllers are identical along $\zW$. 
We thus restrict our analysis to the $\xW$-$\yW$ plane with the corresponding planar cable force $\cableForcePlanar=[\cableForceX\vSpace\cableForceY]^\top$ and guidance force saturation $\positionControllerSatPlanar=[\positionControllerSatX\vSpace\positionControllerSatY]^\top$.
Hereby, the saturation is kept constant such that the human feels a horizontal force of $\norm{\positionControllerSatPlanar}_2=\SI{3}{\newton}$ across all flights. 
Similarly, the admittance controller inertia was set to the constant value of $\inertiaA=5\eye{}$ \si{\kilo\gram}.
However, since the admittance damping $\dampingA$ strongly affects the stability of the closed-loop system (see \theo\ref{thm:stability:velFeedback} and \theo\ref{thm:stability:posCtrlSat:Hvel}), its tuning depends on the performed experiment and is further discussed below.
All parameters and gains have been set to optimize system stability and user comfort. 

\subsection{Tuning of the Admittance Damping \texorpdfstring{$\dampingA$}{s}}
For the remainder of this section, we assume a diagonal admittance and human damping of the form $\dampingA=\dampingAScalar\eye{}$ and $\dampingH=\dampingHScalar\eye{}$.
The parameter $\dampingAScalar$ is the only controller gain that requires adaptation when changing the control law.
Hereby, the tendency for tuning the admittance damping is fundamentally different when using $\positionController$ or $\positionControllerHvel$. 

In the first case, \theo\ref{thm:stability:posCtrlSat} suggests that $\dampingAScalar$ should be chosen as low as possible to allow accurate tracking of $\cableForce$ (see \eqref{eqn:cableForce_posCtrlSat}). 
However, a minimum dissipative action is required to maintain the system stability and ensure a comfortable user experience.
In practice, $\dampingAScalarPosCtrl=\SI{13}{\kilo\gram\per\second}$ proved a good compromise between tracking performance and human comfort. 

On the other hand, the control law $\positionControllerHvel$ intuitively demands $\dampingAScalar$ as high as possible, in order to reach steady-state velocity and thus the desired cable force faster.
As indicated in \theo\ref{thm:stability:posCtrlSat:Hvel} however, a human-damping-dependent upper bound must be respected to maintain stability.
Through iterative trial-and-error, the real-world limits for comfortable operation were found to be $\dampingAScalarPosCtrlHVel=\SI{180}{\kilo\gram\per\second},\dampingAScalarPosCtrlHVel=\SI{100}{\kilo\gram\per\second}$ and $\dampingAScalarPosCtrlHVel=\SI{60}{\kilo\gram\per\second}$ for slow, medium and high human walking speeds, respectively.
It should be noted that these values correspond well with the theoretical limit ($\dampingAScalar<4\dampingHScalar$), as computed in \tab\ref{tab:damping_summary} and visually shown in \fig\ref{fig:damping}.
Hereby, the human damping is approximated as
$\dampingHScalar \approx \frac{\norm{\positionControllerSatPlanar}}{\dpHAvg}$
based on the the human model \eqref{eqn:humanModel},
resulting in the numerical values shown in the table.
However, since the human instinctively increases their internal damping and slows down when the system starts to become unstable (despite the oral feedback to maintain speed), the presented values might underestimate the actual $\dampingHScalar$.
This explains why tuning values above the computed threshold (e.g., for $\dpHAvg=\SI{0.075}{\metre\per\second}$) might still be physically stable.

\begin{figure}
    \centering
    \includegraphics[width=\linewidth]{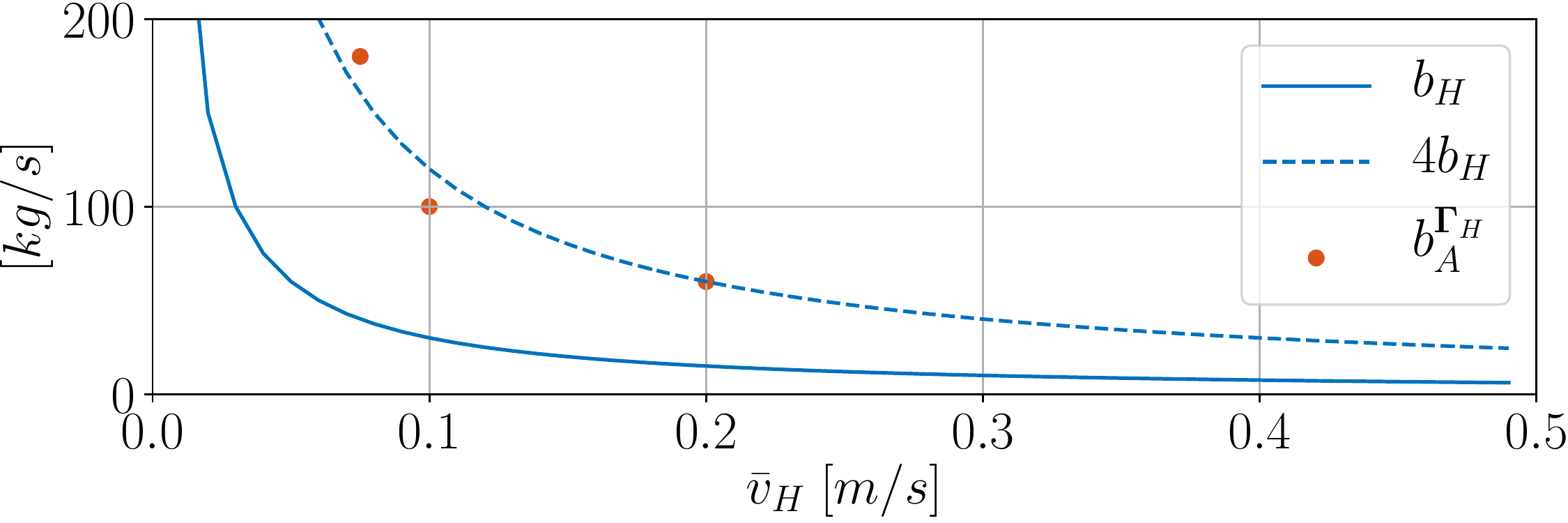}
    \caption{Admittance damping gains and theoretical limit at different walking speeds.}
    \label{fig:damping}
\end{figure}

\begin{table}[t]
    \centering
    \begin{tabular}{c|c|c|c}
    $\dpHAvg\ [\si{\metre\per\second}]$ &
    $\dampingHScalar\ [\si{\kilo\gram\per\second}]$ & 
    $4\dampingHScalar\ [\si{\kilo\gram\per\second}]$ & 
    $\dampingAScalarPosCtrlHVel\ [\si{\kilo\gram\per\second}]$ \\\hline\hline
    0.075 &
    40 &
    160 &
    180\\
    0.1 &
    30 &
    120 &
    100\\
    0.2 &
    15 &
    60 &
    60
    \end{tabular}%
    \caption{Admittance damping gains and theoretical limit at different walking speeds.}%
    \label{tab:damping_summary}%
\end{table}

\subsection{Experimental Results}
The performance of the two control laws $\positionController$ and $\positionControllerHvel$ \add{at steady state} in terms of force error and human velocity at different average walking speeds is shown in \fig\ref{fig:time_force_vel}.
The mean and standard deviation statistics of these quantities over the experiment duration $\expTime$ are summarized in \fig\ref{fig:all_force_vel}. 

\begin{figure*}[h]
    \centering
    \includegraphics[width=\linewidth]{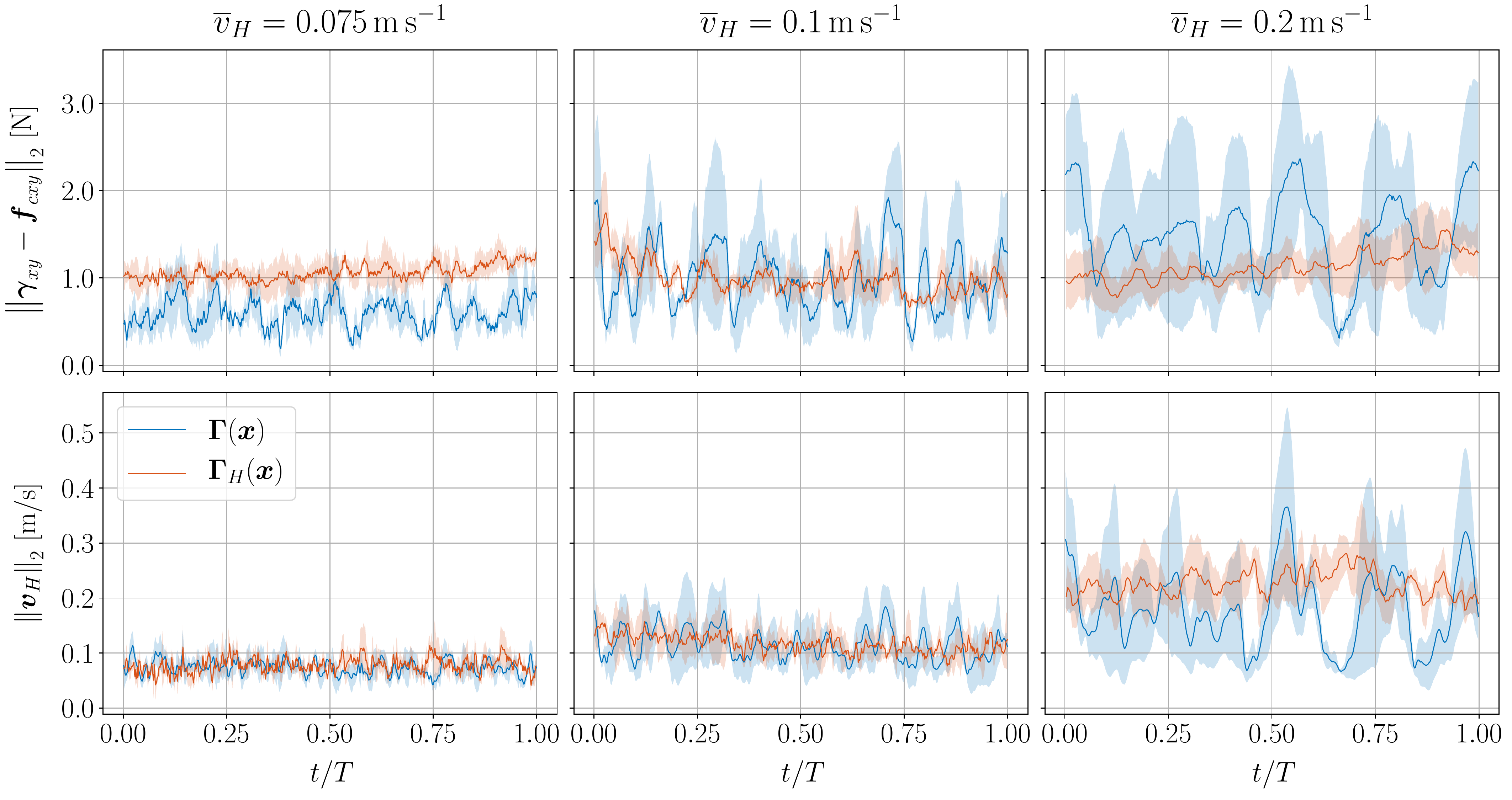}
    \caption{Evolution of force error \add{(top)} and human velocity \add{(bottom)} at different average walking speeds under $\positionController$ and $\positionControllerHvel$ control law \add{at steady-state}. The solid line indicates the average over all test subjects and the transparent area corresponds to the standard deviation.}
    \label{fig:time_force_vel}
\end{figure*}

\begin{figure}[h]
    \centering
    \includegraphics[width=\linewidth]{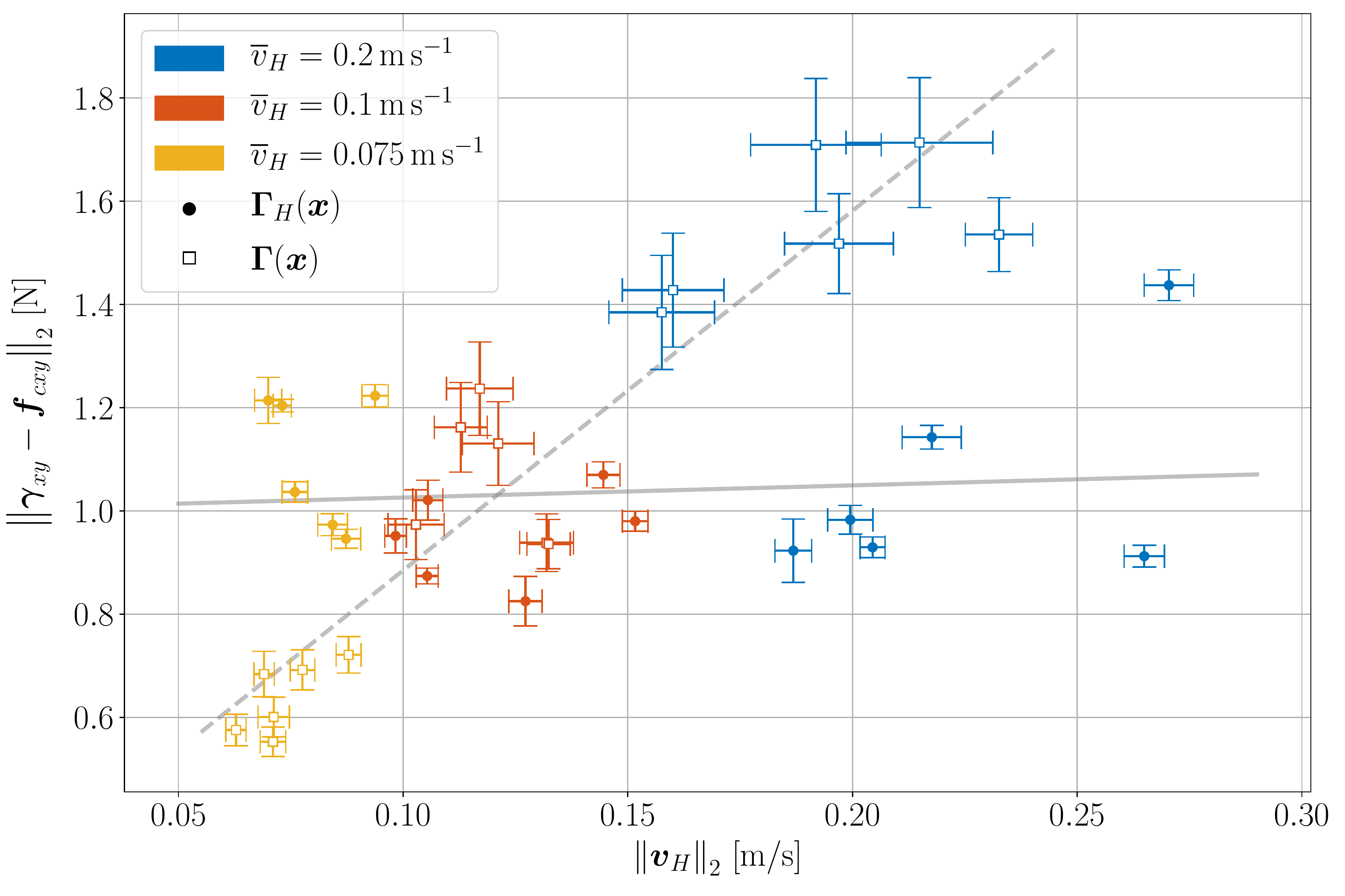}
    \caption{Force error and human velocity mean and $0.125\times$standard deviation (small scale has been used to improve visibility) statistics at different average walking speeds under $\positionController$ and $\positionControllerHvel$ control law \add{at steady state} and computed over the experiment duration $\expTime$. Each data point corresponds to one experiment, i.e., one test subject walking at a specific speed under a given control law. The black solid and dotted line represent the linear interpolation for $\positionControllerHvel$ and $\positionController$, respectively.}%
    \label{fig:all_force_vel}%
\end{figure}%

At low average walking speed, both control laws exhibit similar performance in that the cable tension can be maintained. 
Force error and human velocity show little oscillations, resulting in an efficient and comfortable guidance.

With increasing average walking speed however, larger force errors can be observed in \fig\ref{fig:all_force_vel} when using the controller $\positionController$.
Furthermore, an increase in both velocity and force error standard deviation indicates an increasingly oscillatory behavior.
When the human walks slowly, the  guidance force will raise toward the desired value (force error decreases). However, it in turn makes the human speed-up. 
The robot, having no feedback on the human speed, will keep its current motion, thus making the guiding force reduce again (force error increases). 
This reduction can make the cable go slack, eventually forcing the human to slow down or even stop until tension is re-established. The full cycle will then repeat.
This process is well visible in \fig\ref{fig:time_force_vel}. 
Based on the presented theory and the interpolation line in \fig\ref{fig:all_force_vel}, it is expected that these effects will become even more extreme with increasing walking speeds, i.e. decreasing human damping.

On the other hand, the proposed human velocity feedback controller $\positionControllerHvel$ maintains a consistent force tracking performance across all walking speeds.
Not only that, the almost constant standard deviations shown in \fig\ref{fig:all_force_vel} suggest minimized oscillations in both cable force and human velocity, which is indeed confirmed when looking at the time evolution of both signals in \fig\ref{fig:time_force_vel}.
Overall, this results in a more comfortable user experience compared to $\positionController$.

That being said however, the non-zero force tracking error under the control law $\positionControllerHvel$ \add{(see \fig\ref{fig:all_force_vel} and especially data point at $\norm{\dpH}\approx\SI{0.26}{\metre\per\second}$)} is in contradiction with the theoretical findings. According to \eqref{eqn:cableForce_posCtrlHVelSat}, $\norm{\positionControllerSatPlanar-\cableForcePlanar}_2$ should converge to zero during steady state operation.
This difference can be explained by the fact that the test subjects do not behave exactly like a second order system as it is assumed in the model. 
Notice that \add{small oscillations in human planar velocity due to their gait pattern~\cite{vyas2021}} are not captured.
In practice, high damping gain $\dampingAScalarPosCtrlHVel$ will cause the aerial vehicle to track these human velocities, explaining the observed misalignment between $\positionControllerSatPlanar$ and $\cableForcePlanar$\add{, as well as the small positive slope of the interpolation line in \fig\ref{fig:all_force_vel}.}
Nevertheless, the fact that the offset is \add{nearly} independent of the walking speed still confirms the general intuition.
Notice that this effect does not occur under the controller $\positionController$, as human velocities do not affect the robot motion.




In summary, \fig\ref{fig:all_force_vel} confirms the superiority of the newly proposed guidance law with included human velocity feedback $\positionControllerHvel$.
The cable remains taught at all times for different human walking speeds, therefore being independent of the inherent human damping.
As a result, the human walking pattern is more continuous and does not exhibit the stop-and-go behavior observed with the original guidance law $\positionController$. 
The overall user experience is more comfortable.

\section{CONCLUSIONS}\label{sec:conclusions}
This article discusses the control of a system composed of an aerial robot physically coupled by a cable to the hand of a human being. 
The control problem consists in making the robot capable to safely guide the human to a desired point using forces only as an indirect communication channel.
A first compliant control law that ensure the asymptotic stability of the desired human position was already developed in a previous work.
This paper aims at addressing some of the encountered issues linked to the absence of any feedback nor adaptation to the human-state. 
In particular, we noticed the cable force being less than the desired value when the human walks at a considerable speed.

First, \remove{an equilibrium analysis for the initial control law was performed to formally confirm} the suggested decrease in cable force with increasing human walking speed under saturated control input \add{is formally analyzed}.
\remove{Using Lyapunov theory, the asymptotic convergence to the corresponding steady state was proven.}
\add{In turn, the preexisting controller is modified to include additional feedback w.r.t. the human velocity. This reduces the guiding force error and makes the guidance more efficient. Indeed, formal study of the equilibria and stability confirms asymptotic convergence to zero force error. Furthermore, the system is passive w.r.t. bounded human intent forces and tracking errors due to non-idealities.}
\remove{To improve the force tracking performance and make the guidance more efficient, the preexisting controller is extended with a human velocity feedback.
Given a desired end point, we formally prove the asymptotic stability of the system for this novel controller, provided a human-damping-dependent tuning bound is respected.
Again considering the case when the control force is saturated due to the human being far away from the desired position, asymptotic convergence to a constant velocity trajectory under the same tuning restrictions is shown.
Furthermore, the equilibrium analysis confirms the cable force to reach the desired value at steady state, independently to the human walking speed.}

The theoretical findings are validated in real-world flight tests with a group of three subjects.
Experiments that mimic different human damping confirm the consistent and overall improved force tracking performance of the novel controller.
Cable tension is maintained at all times, thus making the guidance more efficient and comfortable for the human.

Although the proposed control law was successfully evaluated for different human damping values, it currently requires manual retuning if this parameter changes.
Future works will thus include the development of an online estimation and adaptation law to properly detect and react to changes in the human dynamic parameters.
\remove{It is however expected that the extensive testing of such controllers will no longer be possible indoors due to room size limitations. 
We thus foresee the integration of additional sensors for outdoor operation and specifically new methods to estimate the human velocity.}

\section*{ACKNOWLEDGMENT}\label{sec:acknowledgment}
The authors would like to thank L. Streichenberg for his contribution on the software framework and experimental hardware in the early stage of this work.

\bibliographystyle{IEEEtran}
\bibliography{./bibAlias,./bibCustom}

\end{document}